\documentclass{llncs}
\usepackage{amsmath,amssymb,graphicx}
\usepackage{dsfont}
\usepackage{algorithm}
\usepackage{algorithmic}
\usepackage{bm}

\newcommand{\st}{\textstyle}

\newcommand{\N}{\mathbb{N}}

\newcommand{\Tmix}{T_{\rm mix}}

\newcommand{\veps}{\varepsilon}

\newcommand{\ucrl}{\textsc{Ucrl2}}

\newcommand{\meanr}{{r}}

\newcommand{\vect}[1]{\bm{#1}}
\newcommand{\vv}{\vect{v}}
\newcommand{\mat}[1]{\bm{#1}}
\newcommand{\mP}{\mat{P}}
\newcommand{\mI}{\mat{I}}
\newcommand{\tmP}{\ensuremath{\tilde\mP}}
\newcommand{\tpi}[1]{\tilde\pi_{#1}}
\newcommand{\M}{M}

\newcommand{\sM}{\mathcal{M}}
\newcommand{\tM}{\tilde M}
\newcommand{\tMk}{\tM_{k}}
\newcommand{\Ind}[1]{\mathds{1}_{#1}}

\newcommand{\bw}{{\mat w}}

\begin{document} 

\pagestyle{empty}
\title{Regret Bounds for Restless Markov Bandits}
\titlerunning{Regret Bounds for Restless Markov Bandits}
\author{Ronald Ortner\inst{1,2} \and Daniil Ryabko\inst{2} \and Peter Auer\inst{1} \and R{\'e}mi Munos\inst{2}}
\institute{Montanuniversitaet Leoben \and INRIA Lille-Nord Europe, \'{e}quipe SequeL \\
\email{\{rortner,auer\}@unileoben.ac.at, daniil@ryabko.net, remi.munos@inria.fr}
}
\maketitle

\begin{abstract} 
We consider the restless Markov bandit problem, in which the state of each arm evolves
according to a Markov process independently of the learner's actions.  We suggest an algorithm 
that after $T$ steps achieves $\tilde{O}(\sqrt{T})$ regret with respect to the best policy that knows the distributions of all arms. 
No assumptions on the Markov chains are made except that they are irreducible. 
In addition, we show that index-based policies are necessarily suboptimal for the considered problem.
\end{abstract}

\section{Introduction}\label{sec:intro}

In the bandit problem the learner has to decide at time steps $t=1,2,\ldots$ which of the finitely many available arms  to pull.
Each arm produces a reward in a stochastic manner. The goal is to maximize the reward accumulated over time.

Following \cite{laro}, traditionally it is assumed that the rewards produced by each given arm are independent and identically distributed (i.i.d.).
If the probability distributions of the rewards of each arm are known, the best strategy is to only pull the arm with the highest expected reward. 
Thus, in the i.i.d.\ bandit setting the \textit{regret} is measured with respect to the best arm. 
  An extension of this setting is to assume 
that the rewards generated by each arm are not i.i.d., but are governed by some more complex stochastic process.
Markov chains suggest themselves as an interesting and non-trivial model.
In this setting it is often natural to assume that the stochastic process (Markov chain) governing each arm does not depend on the actions 
of the learner. That is, the chain takes transitions independently of whether the learner pulls that arm or not (giving the name {\em restless bandit} to the problem).
The latter property makes the problem rather challenging: since we are not observing the state of each arm, the problem becomes 
a partially observable Markov decision process (POMDP), rather than being a (special case of) a fully observable MDP, as in
 the traditional i.i.d.\ setting.
One of the applications that motivate the restless bandit problem is the so-called {\em cognitive radio} problem (e.g., \cite{akyildiz}): 
Each arm of the bandit is a radio channel that can be busy or available. The learner (an appliance) can only sense a certain 
number of channels (in the basic case only a single one) at a time, which is equivalent to pulling an arm. It is natural to assume 
that whether the channel is busy or not at a given time step depends on the past~--- so a Markov chain is the simplest realistic model~--- but 
does not depend on which channel the appliance is sensing. (See also Example~1 in Section~\ref{sec:ex} for an illustration of a simple instance of this problem.)

What makes the restless Markov bandit problem particularly interesting is that {\em one can do much better than pulling the best arm}. This can be seen already 
on simple examples with two-state Markov chains (see Section~\ref{sec:ex} below).   Remarkably, this feature is often overlooked, 
notably by some early work on restless bandits, e.g. \cite{Anantharam}, where the regret is measured 
with respect to the mean reward of the best arm.  This feature  also makes the problem more difficult and in some sense
more general than the non-stochastic bandit problem, in  which the regret usually is measured with respect to the best arm in hindsight \cite{acfs}.
Finally, it is also this feature that makes the problem principally different from the so-called {\em rested} bandit problem, in which 
each Markov chain only takes transitions when the corresponding arm is pulled. 

 Thus, in the restless Markov bandit problem that we study, the regret should be measured not with respect to the best arm, 
but with respect to the best policy knowing the distribution of all arms. 
To understand what kind of regret bounds can be obtained in this setting, it is useful to compare it to the i.i.d.\ bandit 
problem and to the problem of learning an MDP. In the i.i.d.\ bandit problem, the minimax regret expressed in terms 
of the horizon~$T$ and the number of arms  only is $O(\sqrt T)$, cf.~\cite{aubu}. If we allow problem-dependent constants into consideration, then 
the regret becomes of order $\log T$ but depends also on the gap between the expected reward of the best and the second-best arm.
In the problem of learning to behave optimally in an MDP, nontrivial problem-independent finite-time  regret guarantees (that is,
regret depending only on $T$ and the number of states and actions) are not possible to achieve. It is possible to 
obtain $O(\sqrt T)$ regret bounds that also depend on the diameter of the MDP \cite{jaorau} or similar related constants, such as the span of the optimal bias vector \cite{regal}.
Regret bounds of order $\log T$ are only possible if one additionally allows into consideration constants expressed in terms of policies, such as the gap between the average reward obtained by the best 
and the second-best policy \cite{jaorau}. The difference between these constants and constants such as the diameter of an MDP is that 
one can try to estimate the latter, while estimating the former is at least as difficult as solving the original problem~--- finding the best policy.
Turning to our restless Markov bandit problem, so far, to the best of our knowledge no regret bounds are available for the general 
problem. However, several special cases have been considered. Specifically, $O(\log T)$ bounds have been  
obtained in \cite{teli2} and \cite{filippi}. While the latter considers the two-armed restless bandit case, 
the results of~\cite{teli2} are constrained by some ad hoc assumptions on the transition probabilities
and on the structure of the optimal policy of the problem. Also the dependence of the regret bound on the problem parameters
is unclear, while computational aspects of the algorithm (which alternates exploration and exploitation steps) are neglected.
Finally, while regret bounds for the Exp3.S algorithm \cite{acfs} could be applied, these depend on the ``hardness'' of 
the reward sequences, which in the case of reward sequences generated by a Markov chain can be arbitrarily high.

Here we present an algorithm for which we derive $\tilde{O}(\sqrt{T})$
regret  bounds, making no assumptions on the distribution of the Markov chains. 
The algorithm is based on constructing an approximate MDP representation of the POMDP problem, 
and then using a modification of the \ucrl\ algorithm of \cite{jaorau} to learn this approximate MDP. 
 In addition to the horizon~$T$ and the number of arms and states, the regret bound also depends 
 on the diameter and the mixing time (which can be eliminated however) of the Markov chains of the arms.
 If the regret has to be expressed only in these terms, then our lower bound shows that the 
 dependence on $T$ cannot be significantly improved.

\section{Preliminaries}\label{sec:setting}
Given are $K$ arms, where underlying each arm $j$ there is an irreducible Markov chain with state space $S_j$ 
and transition matrix $P_j$. For each state $s$ in $S_j$ there are mean rewards $r_j(s)$, which we assume to be bounded in $[0,1]$.
For the time being, we will assume that the learner knows the number of states for each arm
and that all Markov chains are aperiodic.
In Section~\ref{sec:proofs}, we discuss periodic chains, while in Section~\ref{sec:unknownS} 
we indicate how to deal with unknown state spaces.
In any case, the learner knows neither the transition probabilities nor the mean rewards. 

For each time step $t=1,2,\ldots$ the learner chooses one of the arms, observes the current state $s$
of the chosen arm $i$ and receives a random reward with mean~$r_i(s)$. After this, the state of each arm~$j$
changes according to the transition matrices $P_j$. The learner however is not able to observe the 
current state of the individual arms.
We are interested in competing with the optimal policy $\pi^*$ which knows the mean rewards and transition
matrices, yet observes as the learner only the current state of the chosen arm. 
Thus, we are looking for 
algorithms which after any $T$ steps have small regret with respect to $\pi^*$, i.e.\ minimize
\[
   \st T\cdot \rho^* - \sum_{t=1}^T r_t,
\]
where $r_t$ denotes the (random) reward earned at step $t$ and $\rho^*$ is the average reward of the optimal policy $\pi^*$.
(It will be seen in Section~\ref{sec:alg} that $\pi^*$ and $\rho^*$ are indeed  well-defined.) 

\subsubsection{Mixing Times and Diameter}
If an arm $j$ is not selected for a large number of time steps, the distribution over states 
when selecting $j$ will be close to the stationary distribution $\mu_j$ of the Markov chain underlying arm $j$.
Let $\mu_{s}^t$ be the distribution after $t$ steps when starting in state $s\in S_j$. Then setting
\[
    d_j(t):= \max_{s\in S_j} \|\mu_{s}^t - \mu_j\|_1:= \max_{s\in S_j} \sum_{s'\in S_j} |\mu_{s}^t(s') - \mu_j(s')|,
\]
we define the \textit{$\veps$-mixing time} of the Markov chain as
\[
    T^j_{\rm mix}(\veps) := \min \{t \in \N \,|\, d_j(t)\leq \veps \}.
\]
Setting somewhat arbitrarily \textit{the} mixing time of the chain to 
$T^j_{\rm mix}:=T^j_{\rm mix}(\frac14)$, one can show (cf.~eq.~4.36 in \cite{lepewi}) that 
\begin{equation}\label{eq:tmix}
    T^j_{\rm mix}(\veps)  \leq \left\lceil \log_2 \tfrac{1}{\veps} \right\rceil \cdot T^j_{\rm mix}.
\end{equation}
Finally, let $T_j(s,s')$ be the expected time it takes in arm $j$ to reach $s'$ when starting in $s$.
We set the \textit{diameter} of arm $j$ to be $D_j:=\max_{s,s'\in S_j} T_j(s,s')$.

\section{Examples}\label{sec:ex}
Next we present a few examples that give insight into the nature of the problem and the difficulties 
in finding solutions. In particular, the examples demonstrate that (i) the optimal reward can be (much) bigger
than the average reward of the best arm, (ii) the optimal policy does not maximize the immediate reward, (iii) 
the optimal policy cannot always be expressed in terms of arm indexes.

\begin{example}\label{ex:sym} In this example the average reward of each of the two arms of a bandit is $\frac12$, but the reward of the optimal 
policy is close to $\frac34$.
 Consider a two-armed bandit. Each arm has two possible states, 0 and 1, which are also the rewards.
Underlying each of the two arms is a (two-state) Markov chain with transition matrix 
$\left(\begin{array}{cc}
  1-\epsilon&\epsilon\\
  \epsilon&1-\epsilon
 \end{array}\right)
$, where $\epsilon$ is small. Thus, a typical trajectory of each arm looks like this:
$
000000000001111111111111111000000000\dots,
$
and the average reward for each arm is $\frac12$. It is easy to see that the optimal policy starts with any arm, and then switches the arm whenever the reward is 0, and otherwise sticks to the same arm. The average reward is close to $\frac34$~--- much larger than the
reward of each arm.

This example has a natural interpretation in terms of {\em cognitive radio}: two radio channels are available, each of which can be either busy (0) or available~(1).
A device can only sense (and use) one channel at a time, and one wants to maximize the amount of time the channel it tries to use is available. 
\end{example}
\begin{example}\label{ex:sym2}
 Consider the previous example, but with $\epsilon$ close to 1. Thus, a typical trajectory of each arm is now 
$
0101010100101	0110\dots,
$ and the optimal policy switches arms if the previous reward was 1 and stays otherwise.
\end{example}

\begin{example}\label{ex:asym} In this example the optimal policy does not  maximize the immediate reward.
Again, consider a two-armed bandit. Arm 1 is as in Example~\ref{ex:sym}, and arm~2 provides Bernoulli i.i.d.\ rewards
with probability $\frac12$ of getting reward 1. The optimal policy (which knows the distributions) will sample arm 1 until it obtains reward~0, when it
switches to arm 2. However, it will sample arm 1 again after some time $t$ (depending on $\epsilon$), and 
only switch back to arm~2 when the reward on arm 1 is 0. Note that
whatever $t$ is, the expected reward for choosing arm~1 will be strictly smaller than $\frac12$, since the last observed reward was 0 
and the limiting probability of observing reward 1 (when $t\to\infty$) is~$\frac12$. At the same time, the expected reward of the second arm is always $\frac12$.
Thus, the optimal policy will sometimes ``explore'' by pulling the arm with the smaller expected reward.
\end{example}

An intuitively appealing idea is to look for an optimal policy in an {\em index} form. That is, for each
arm the policy maintains an index which is a function of time, states, and rewards {\em of this arm only}.
At each time step, the policy samples the arm that has maximal index. This seems promising for at least two reasons: First,
the distributions of the arms are assumed independent, so it may seem reasonable to evaluate them independently as well; 
second, this works in the i.i.d.\ case (e.g., the Gittins index \cite{Gittins} or UCB \cite{acbf}). This idea also motivates 
the setting when just one out of two arms is Markov and the other is i.i.d., see e.g.~\cite{filippi}.
Index policies for restless Markov bandits were also studied in \cite{whittle}.
Despite their intuitive appeal, in general, index policies are suboptimal.
\begin{theorem}\label{thm:index}
 For each index-based policy $\pi$ there is a restless Markov bandit problem in which $\pi$ behaves suboptimally.
\end{theorem}
\begin{proof}
 Consider the three bandits L (left), C (center), and R (right) in Figure~\ref{fig:cex},
 where C and R start in the 1 reward state. (Arms $C$ and $R$ can easily be made aperiodic 
 by adding further sufficiently small transition probabilities.) 
 \begin{figure}[t]%
\centering
\scalebox{0.1}{\includegraphics{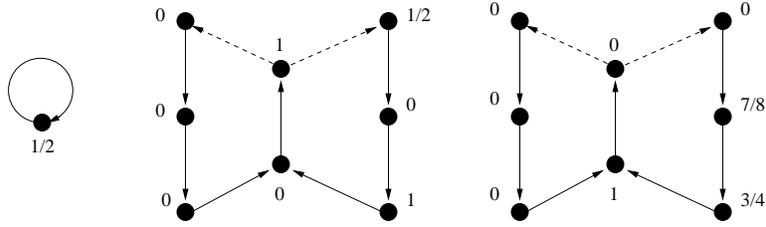}}
\caption{{\em Example 4.} 
 Dashed transitions are with probability $\frac12$, others are deterministic with probability 1.
Numbers are rewards in the respective state.}\label{fig:cex}
\end{figure}
 Assume that C has been observed in the $\frac12$ reward
 state one step before, while R has been observed in the 1 reward state three steps ago.
 The optimal policy will choose arm L which gives reward~$\frac12$ with certainty 
 (C gives reward 0 with certainty, while R gives reward~$\frac78$ with probability $\frac12$)
 and subsequently arms C and R.
 However, if arm C was missing, in the same situation, the optimal policy would choose~R: 
 Although the immediate expected reward is smaller than when choosing L, sampling R
 gives also information about the current state, which can earn reward $\frac34$ a step later.
Clearly, no index based policy will behave optimally in both settings.
\qed
\end{proof}

\section{Main Results}\label{sec:res}
\begin{theorem}\label{thm:regret}
Consider a restless bandit with $K$ aperiodic arms having state spaces~$S_j$, diameters $D_j$, and
mixing times $\Tmix^j$ ($j=1,\ldots,K$). Then
with probability at least $1-\delta$ the regret of Algorithm~\ref{alg:gen} (presented in Section~\ref{sec:alg} below) after $T$ steps is upper bounded by
\[  {\rm const} \cdot S \cdot \Tmix^{3/2} \cdot {\st\prod_{j=1}^K}(4D_j) \cdot \max_i \log (D_i) \cdot \log^{2} \big(\tfrac{T}{\delta}\big) \cdot
       \sqrt{T},  \]
       where $S:=\sum_{j=1}^K |S_j|$ is the total number of states and $\Tmix:=\max_j \Tmix^j$ the maximal mixing time. 
Further, the dependence on $\Tmix$ can be eliminated to show that with probability at least $1-\delta$ the regret is bounded by
\[ {O}\left(  S \cdot \mbox{$\prod_{j=1}^K$}(4D_j)\cdot \max_i \log (D_i) \cdot \log^{7/2} \big(\tfrac{T}{\delta}\big) 
      \cdot \sqrt{T} \right).  \]
\end{theorem}

\begin{remark}
For periodic chains the bound of Theorem~\ref{thm:regret} has worse dependence on the state space, for details 
see Remark~\ref{rem:periodic} in Section~\ref{sec:proofs}.
\end{remark}

\begin{theorem}\label{thm:lobo}
 For any algorithm, any $K>1$ and any $m\geq 1$ there is a $K$-armed restless bandit problem with a total number of $S:=Km$ states, such 
 that the regret after $T$ steps is lower bounded by $\Omega(\sqrt{ST})$.
\end{theorem}

\begin{remark}
 While it is easy to see that lower bounds depend on the total number of states over all arms,
 the dependence on other parameters in our upper bound is not clear.
 For example, intuitively, while in the general MDP case one wrong step may cost up to $D$ --- the MDP's diameter \cite{jaorau} --- 
 steps to compensate for,
 here the Markov chains evolve independently of the learner's actions, and the upper bound's dependence on the diameter 
 may be just an artefact of the proof.
\end{remark}

\section{Constructing the Algorithm}\label{sec:alg}

\subsubsection{MDP Representation}
We represent the setting as an MDP by recalling for each arm the last observed state and the number of time steps
which have gone by since this last observation. Thus, each state of the MDP representation is of the form $(s_{j},n_j)_{j=1}^K:=(s_{1},n_1,s_{2},n_2,\ldots,s_{K},n_K)$
with $s_j\in S_j$ and $n_j\in \N$,
meaning that each arm $j$ has not been chosen for $n_j$ steps when it was in state~$s_{j}$.
More precisely, $(s_{j},n_j)_{j=1}^K$ is a state of the considered MDP if and only if 
(i)~all~$n_{j}$ are distinct and (ii) there is a $j$ with $n_{j}=1$.\footnote{Actually, one would need 
to add for each arm $j$ with $|S_j|>1$ a special state for not having sampled $j$ so far. 
However, for the sake of simplicity we assume that in the beginning each arm is sampled once. The respective regret is negligible.} 
The action space of the MDP is $\{1,2,\ldots,K\}$, and the transition probabilities from a state $(s_{j},n_j)_{j=1}^K$
are given by the $n_j$-step transition probabilities $p_j^{(n_j)}(s,s')$ of the Markov chain underlying the chosen arm $j$ 
(these are defined by the matrix power of the single step transition probability matrix, i.e.\ $P^{n_j}_j$).
That is, the probability for a transition from state $(s_{j},n_j)_{j=1}^K$ to $(s'_{j},n'_j)_{j=1}^K$ under action $j$ 
is given by $p_j^{(n_j)}(s_{j},s'_{j})$ iff (i)~$n'_j=1$, (ii) $n'_\ell=n_\ell+1$ and $s_{\ell}=s'_{\ell}$ for all $\ell\neq j$.
All other transition probabilities are 0.
Finally, the mean reward for choosing arm $j$ in state $(s_{j},n_j)_{j=1}^K$ 
is given by $\sum_{s\in S_j} p_j^{(n_j)}(s_{j},s)\cdot r_j(s)$.
This MDP representation has already been considered in~\cite{teli2}. 

Obviously, within $T$ steps any policy can reach only states with $n_j\leq T$.
Correspondingly, if we are interested in the regret within $T$ steps, it will be
sufficient to consider the finite sub-MDP consisting of states with $n_j\leq T$.
We call this the \textit{$T$-step representation} of the problem, and the regret 
will be measured with respect to the optimal policy in this $T$-step representation.

\subsubsection{Structure of the MDP Representation}

The MDP representation of our problem has some special structural properties.
In particular, rewards and transition probabilities for choosing arm $j$ only
depend on the state of arm $j$, i.e.\ $s_j$ and $n_j$. Moreover, the support 
for each transition probability distribution is bounded, 
and for $n_j\geq \Tmix^j(\veps)$ the transition probability distribution will be close 
to the stationary distribution of arm $j$. 
%
%
\begin{algorithm}[tb]
   \caption{The colored \ucrl\ algorithm}
   \label{alg:col}
\begin{algorithmic}
    \STATE   {\bfseries Input:} Confidence parameter $\delta>0$, aggregation parameter $\veps>0$, state space $S$, action space $A$, coloring and translation functions, a bound $B$ on the size of the support of transition probability distributions.\smallskip
    \STATE   {\bfseries Initialization:}  Set $t:=1$, and observe the initial state $s_1$.\smallskip
    \FOR {episodes $k=1,2,\ldots$} 
        \STATE \textbf{Initialize episode }$k$:\\ 
             Set the start time of episode $k$, $t_k:=t$.
               Let $N_{k}\left (c \right)$ be the number of times a state-action pair of color $c$ has been visited prior to episode~$k$, 
               and $v_k(c)$ the number of times a state-action pair of color $c$ has been visited in episode $k$. 
                Compute estimates $\hat{r}_{k}(s,a)$ and $\hat{p}_{k}(s'|s,a)$ for rewards and transition probabilities, 
                using all samples from state-action pairs of the same color $c(s,a)$, respectively.\smallskip
        \STATE \textbf{Compute policy} $\tpi{k}$:\\ 
        	Let $\mathcal M_k$ be the set of plausible MDPs	with rewards $\tilde r(s, a)$ and 
        	transition probabilities $\tilde{p}(\cdot|s,a)$ satisfying
	\begin{eqnarray}\textstyle
          \label{eq:civR}
             \big\vert
              \tilde r(s, a) - \hat{r}_k(s,a)
            \big \vert \;\;
          & \leq & \;
           \veps + \sqrt{ \tfrac{7\log\left( 2 C t_k / \delta  \right )}
              {2\max\{1,N_k(c(s,a))\}},
            }
          \\
          \label{eq:civ}
           \Big\Vert
             \tilde{p}(\cdot|s,a) - \hat{p}_k(\cdot|s,a)
            \Big \Vert_1
          & \leq &\;
             \veps + \sqrt{ \tfrac{56B\log\left( 4 C t_k / \delta  \right )}
              {\max\{1,N_k(c(s,a))\}}
            }
          \;,
          \end{eqnarray}
          where $C$ is the number of distinct colors. Let $\rho(\pi,M)$ be the average reward of a policy 
          $\pi: S\to A$ on an MDP $M\in\mathcal M_k$. Choose (e.g.\ by extended value iteration~\cite{jaorau})
          an optimal policy $\tpi{k}$ and an optimistic $\tMk\in\mathcal M_k$ such that
          \begin{equation}\label{eq:opt}
          \rho(\tpi{k},\tMk)=\max\{\rho(\pi,M)\,|\,\pi: S\to A,\, M\in\mathcal M_k\}.
          \end{equation}
        \STATE \textbf{Execute policy $\tpi{k}$}:
            \STATE \textbf{while} $v_k(c(s_{t},\tpi{k}(s_t))) < \max\{1,N_{k}(c(s_{t},\tpi{k}(s_{t})))\}$ \textbf{do} 
                 \STATE $\rhd$ Choose action $a_{t} = \tpi{k}(s_{t})$, 
                 obtain reward $r_t$, and observe next state~$s_{t+1}$.	
                 \STATE $\rhd$ Set $t:=t+1$.
             \STATE \textbf{end while}
        \ENDFOR
    \end{algorithmic}
\end{algorithm}
Thus, one could reduce the $T$-step representation further by aggregating states\,\footnote{\label{fn:agg}Aggregation of states $s_1,\ldots,s_n$ means that these states are replaced by a new state~$s_{\rm agg}$ inheriting rewards and transition probabilities from an arbitrary $s_i$ (or averaging over all $s_j$). Transitions to this state are set to 
$p(s_{\rm agg}|s,a):=\sum_{j} p(s_j|s,a)$.} 
$(s_{j},n_j)_{j=1}^K$, $(s'_{j},n'_j)_{j=1}^K$
whenever $n_j,n'_j \geq \Tmix^j(\veps)$ and $s_\ell=s'_\ell$, $n_\ell=n'_\ell$ for $\ell\neq j$.
The rewards and transition probability distributions of aggregated states are $\veps$-close,
so that the error by aggregation can be bounded by results given in~\cite{or-alt07}.
While this is helpful for approximating the problem when all parameters are known,
it cannot be used directly when learning, since the observations in the aggregated states do
not correspond to an MDP anymore. Thus, while standard reinforcement learning algorithms are still applicable,
there are no theoretical guarantees for them.

\subsubsection{$\veps$-structured MDPs and Colored UCRL2}

In the following, we exploit the special structure of the MDP representation. 
We generalize some of its structural properties in the following definition.

\begin{definition}
An $\veps$-structured MDP is an MDP with finite state space $S$, finite action space $A$, transition probability distributions $p(\cdot|s,a)$, mean rewards $r(s,a)\in[0,1]$, and a coloring function $c:S\times A \to \mathcal{C}$, where $\mathcal{C}$ is a set of colors.
 Further, for each two pairs $(s,a)$, $(s',a')\in S\times A$ with $c(s,a)=c(s',a')$ there is a bijective translation function $\phi_{s,a,s',a'}:S\to S$ such that $\sum_{s''}\big|p(s''|s,a) - p(\phi_{s,a,s',a'}(s'')|s',a')\big| < \veps$ and $|r(s,a)-r(s',a')|<\veps$.
\end{definition}

If there are states $s,s'$ in an $\veps$-structured MDP such that $c(s,a)=c(s',a)$ for all actions $a$ and
the associated translation function $\phi_{s,a,s',a}$ is the identity, we may aggregate the states (cf.~footnote~\ref{fn:agg}).
We call the MDP in which all such states are aggregated the \textit{aggregated $\veps$-structured MDP}.\smallskip

For learning in $\veps$-structured MDPs we consider a modification of the \ucrl\ algorithm of~\cite{jaorau}.
The \textit{colored \ucrl} algorithm is shown in Figure \ref{alg:col}. As the original \ucrl\ algorithm it 
maintains confidence intervals for rewards and transition probabilities which define a set of plausible MDPs $\mathcal M_k$.
In each episode $k$, the algorithm chooses an optimistic MDP $\tMk\in\mathcal M_k$ and an optimal policy which maximize
the average reward, cf.~\eqref{eq:opt}. Colored \ucrl\ calculates estimates from all samples of state-action pairs of the 
same color, and works with respectively adapted confidence intervals and a corresponding adapted episode termination criterion.
Basically, an episode ends when for some color $c$ the number of visits in state-action pairs of color $c$ has doubled.

\begin{algorithm}[tb]
   \caption{The restless bandits algorithm}
   \label{alg:gen}
\begin{algorithmic}
    \STATE   {\bfseries Input:} Confidence parameter $\delta>0$, the number of states $S_j$ and mixing time
    $\Tmix^j$ of each arm $j$, horizon $T$. \smallskip
    \STATE $\rhd$ Choose $\veps=1/\sqrt{T}$ and execute colored \ucrl\ (with confidence parameter $\delta$) on 
 	the $\veps$-structured MDP described in the ``coloring'' paragraph at the end of Section~\ref{sec:alg}.
\end{algorithmic}
\end{algorithm}

\subsubsection{Coloring the $T$-step representation}
Now, we can turn the $T$-step representation into an $\veps$-structured
MDP, assigning the same color to state-action pairs where the chosen
arm is in the same state, that is, $c((s_i,n_i)_{i=1}^K, j)= c((s'_i,n'_i)_{i=1}^K, j')$ iff
$j=j'$, $s_j=s'_j$, and either $n_j=n'_j$ or $n_j,n'_j\geq \Tmix^j(\veps)$.
The translation functions are chosen accordingly.
This $\veps$-structured MDP can be learned with colored \ucrl, see Algorithm~\ref{alg:gen}, 
our restless bandits algorithm.
(The dependence on the horizon $T$ and the mixing times $\Tmix^j$ as input parameters can be 
eliminated, cf.\ the proof of Theorem~\ref{thm:regret} in Section~\ref{sec:proofs}.)

\section{Regret Bounds for Colored UCRL2}
The following is a generalization of the regret bounds for \ucrl\ to 
$\veps$-structured MDPs. The theorem gives improved (with respect to \ucrl) 
 bounds if there are only a few parameters to estimate in the MDP to learn.
Recall that the \textit{diameter} of an MDP is the maximal expected transition time between any two states (choosing an appropriate policy), 
cf.~\cite{jaorau}.

\begin{theorem}\label{thm:aggsamples}
Let $M$ be an $\veps$-structured MDP with finite state space $S$, finite action space $A$, transition probability distributions $p(\cdot|s,a)$, mean rewards $r(s,a)\in[0,1]$, coloring function $c$ and
associate translation functions. Assume the learner has complete knowledge of state-action pairs $\Psi_K\subseteq S\times A$, while the state-action pairs in $\Psi_U:=S\times A \setminus \Psi_K$ are unknown and have to be learned. However, the learner knows $c$ and all associate translation functions as well as an upper bound~$B$ on the size of the support of each transition probability distribution in~$\Psi_U$.
Then with probability at least $1-\delta$, after any $T$ steps colored \ucrl\;\footnote{For the sake of simplicity the algorithm was given for the case $\Psi_K=\varnothing$. It is obvious how to extend the algorithm when some parameters are known.} gives regret upper bounded by
\[
   42 D_\veps \sqrt{B C_U T \log\big(\tfrac{T}{\delta}\big)} + \veps (D_\veps+2) T,
\]
where $C_{U}$ is the total number of colors for states in $\Psi_U$, and $D_\veps$ is the diameter of the aggregated $\veps$-structured MDP.
\end{theorem}
The proof of this theorem is given in the appendix.
\begin{remark}\label{rem:eps=0}
For $\veps=0$, one can also obtain logarithmic bounds analogously to Theorem~4 of \cite{jaorau}. 
With no additional information for the learner one gets the original \ucrl\ bounds (with a slightly larger constant), 
trivially choosing $B$ to be the number of states and assigning each state-action pair an individual color.
\end{remark}

\section{Proofs}\label{sec:proofs}
We start with bounding the diameter in the aggregated $\veps$-structured MDP.
\begin{lemma}\label{prop:Deps}
For $\veps\leq 1/4$, the diameter $D_\veps$ in the aggregated $\veps$-structured MDP can be upper bounded by 
$2\, \big\lceil \log_2(4 \max_j D_j)\big\rceil \cdot \Tmix(\veps)\cdot \prod_{j=1}^K (4D_j)$,
where we set  $\Tmix(\veps):=\max_j \Tmix^j(\veps)$.
\end{lemma}
\begin{proof}
Let $\mu_j$ be the stationary distribution of arm $j$. 
It is well-known that the expected \textit{first return time} $\tau_j(s)$ in state $s$ satisfies $\mu_j(s)=1/\tau_j(s)$.
Set $\tau_j:=\max_{s} \tau_j(s)$, and $\tau:=\max_j \tau_j$. Then, $\tau_j \leq 2 D_j$.

Now consider the following scheme to reach a given state $(s_j,n_j)_{j=1}^K$:
First, order the states $(s_j,n_j)$ descendingly with respect to $n_j$. 
Thus, assume that  $n_{j_1}>n_{j_2}>\ldots>n_{j_K}=1$.
Take $\Tmix(\veps)$ samples from arm~$j_1$. (Then each arm will be $\veps$-close to the stationary distribution,
and the probability of reaching the right state $s_{j_i}$ when sampling arm $j_i$ afterwards is at least $\mu_{j_i}(s_{j_i})-\veps$.)
Then sample each arm $j_2,j_3,\ldots$ exactly $n_{j_{i-1}}-n_{j_i}$ times.

We first show the lemma for $\veps\leq \mu_{0}:=\min_{j,s} \mu_j(s)/2$.
As observed before, for each arm $j_i$ the probability of reaching the right state $s_{j_i}$ 
is at least $\mu_{j_i}(s_{j_i})-\veps \geq \mu_{j_i}(s_{j_i})/2$.
Consequently, the expected number of restarts of the scheme necessary to reach a particular state $(s_j,n_j)_{j=1}^K$ 
is upper bounded by $\prod_{j=1}^K 2/\mu_j(s_j)$. 
As each trial takes at most $2\Tmix(\veps)$ steps, recalling that $1/\mu_j(s)=\tau_j(s)\leq 2D_j$ proves the bound
for $\veps\leq \mu_{0}$.

Now assume that $\veps>\mu_0$. Since $D_{\veps}\leq D_{\veps'}$ for $\veps > \veps'$
we obtain a bound of $2\Tmix(\veps') \prod_{j=1}^K (4D_j)$ with $\veps':=\mu_0=1/2\tau$.
By \eqref{eq:tmix}, we have $\Tmix(\veps')\leq$\linebreak $\lceil \log_2 (1/\veps')\rceil \, \Tmix(1/4) \leq \lceil \log_2 (4\tau)\rceil \,\Tmix(\veps)$,
which proves the lemma.
\qed
\end{proof}\smallskip

\noindent
{\textbf{Proof of Theorem \ref{thm:regret}.}} 
Note that in each arm~$j$ the support of the transition probability distribution is upper bounded by $|S_j|$.
Hence, Theorem~\ref{thm:aggsamples} with $C_U=\sum_{j=1}^K |S_j|\,\Tmix^j(\veps)$ 
and $B=\max_j |S_j|$ shows that  the regret is bounded by
$42 D_\veps  \sqrt{\max_i |S_i| \cdot {\st\sum_{j=1}^K} |S_j|\cdot\Tmix^j(\veps) \cdot T \log\big(\tfrac{T}{\delta}\big)} 
          		+ \veps (D_\veps+2) T$
with probability $\geq 1-\delta$.
Since $\veps=1/\sqrt{T}$, this proves the first bound by Lemma~\ref{prop:Deps} and recalling~\eqref{eq:tmix}.

If the horizon $T$ is not known, guessing $T$ using the doubling trick 
(i.e., executing the algorithm for $T=2^i$ with confidence parameter $\delta/2^i$ in rounds $i=1,2,\ldots$) 
achieves the bound given in Theorem~\ref{thm:regret}
 with worse constants. 
 
 Similarly, if $\Tmix$ is unknown, one can perform the algorithm in rounds $i=1,2,\ldots$ of length $2^i$ with confidence parameter $\delta/2^i$, 
 choosing an increasing function $a(t)$ to guess an upper bound on $\Tmix$ at the beginning~$t$ of each round.
  This gives a bound of order 
 $a(T)^{3/2}\sqrt{T}$ with a corresponding additive constant. In particular, choosing $a(t)=\log t$
the regret is bounded by
 ${O}\big(S \cdot \prod_{j=1}^K(4D_j)\cdot\max_i\log(D_i)\cdot \log^{7/2}(T/\delta)\cdot\sqrt{T}\big)$
 with probability $\geq 1-\delta$.
\qed

\begin{remark}\label{rem:reversible}
Whereas it is not easy to obtain upper bounds on the mixing time in general, 
for \textit{reversible} Markov chains $\Tmix$ can be linearly upper bounded by the diameter, 
cf.~Lemma 15 in Chapter 4 of \cite{alfill}. While it is possible to compute an upper 
bound on the diameter of a Markov chain from samples of the chain, we did not 
succeed in deriving any useful results on the quality of such bounds.
\end{remark}

\begin{remark}\label{rem:periodic}
 Periodic Markov chains do not converge to a stationary distribution. However taking into account 
the period of the arms, one can generalize our results to the periodic case. Considering
in an $m$-periodic Markov chain the $m$-step transition probabilities given by the matrix $P^m$, one obtains
$m$ distinct aperiodic chains (depending on the initial state) each of which converges to a stationary distribution with respective
mixing times. The maximum over these mixing times can be considered to be \textit{the} mixing time of the chain.

Thus, instead of aggregating states $(s_j,n_j)$, $(s'_j,n'_j)$ with $n_j,n'_j\geq \Tmix^j(\veps)$ as in the case 
of aperiodic chains, one aggregates them only if additionally $n_j\equiv n'_j \mod m_j$.
If the periods $m_j$ are not known to the learner, one can use the least common denominator of
$1,2,\ldots,|S_j|$ as period. Since by the prime number theorem the latter is exponential in $|S_j|$,
the obtained results for periodic arms show worse dependence on the number of states.
(Concerning the proof of Lemma~\ref{prop:Deps} the sampling scheme has to be slightly adapted 
so that one samples in the right period when trying to reach a particular state.)
\end{remark}

\noindent\textbf{Proof of Theorem~\ref{thm:lobo}.}
   Consider $K$ arms 
 all of which are deterministic cycles of length $m$ and hence $m$-periodic. Then the learner
 faces $m$ distinct learning problems with $K$ arms, each of which can be made to force regret of order $\Omega(\sqrt{KT/m})$   
 in the $T/m$ steps the learner deals with the problem \cite{acfs}. Overall, this gives the claimed bound of $\Omega(\sqrt{mKT})=\Omega(\sqrt{ST})$.
 Adding a sufficiently small probability (with respect to the horizon $T$) of staying in some state of each arm,
 one obtains the same bounds for aperiodic arms.
\qed

\section{Extensions and Outlook}
\noindent{\bf Unknown state space.} \label{sec:unknownS}
 If (the size of) the state space of the individual arms is unknown,
some additional exploration of each arm will sooner or later determine the state space.
Thus, we may execute our algorithm on the known state space where between two episodes
we sample each arm until all known states have been sampled at least once.
The additional exploration is upper bounded by $O(\log T)$, as there are only 
$O(\log T)$ many episodes, and the time of each exploration phase can be bounded
with known results.
 That is, the expected number of exploration steps needed 
until all states of an arm $j$ have been observed is upper bounded by $D_j \log (3|S_j|)$ (cf.\ Theorem~11.2 of~\cite{lepewi}),
while the deviation from the expectation can be dealt with by Markov inequality or results from~\cite{ald}. 
That way, one obtains bounds as in Theorem~\ref{thm:regret} for the case of unknown state space.\smallskip

\noindent{\bf Improving the bounds.}
All parameters considered, there is still a large gap between the lower and the upper bound on the regret.
As a first step, it would be interesting to find out whether the dependence
on the diameter of the arms is necessary.
Also, the current regret bounds do not make use of the interdependency of the 
transition probabilities in the Markov chains and treat $n$-step and $n'$-step 
transition probabilities independently. 
Finally, a related open question is how to obtain estimates and upper bounds on mixing times.\smallskip

 \noindent{\bf More general models.} After considering bandits with  i.i.d.\ and Markov arms, the next natural step is to 
 consider more general time-series distributions. Generalizations are not straightforward: already for the case of   Markov 
 chains of order (or memory) 2 the MDP representation of the problem (Section~\ref{sec:alg}) breaks down, and so the approach 
 taken here cannot be easily extended. Stationary ergodic distributions are an interesting more general case, for which
the first question is whether it is possible  to obtain asymptotically sublinear regret.

\subsubsection*{Acknowledgments.}
{\small This research 
was funded by the Ministry of Higher Education and Research, Nord-Pas-de-Calais Regional Council and 
 FEDER (Contrat de Projets Etat Region CPER  2007-2013),
  ANR  projects EXPLO-RA (ANR-08-COSI-004),
  Lampada (ANR-09-EMER-007)  and CoAdapt,
  and by the European Community's  FP7 Program  under grant agreements 
 n$^\circ$\,216886 (PASCAL2) and n$^\circ$\,270327 (CompLACS).
 The first author is currently funded by the Austrian Science Fund (FWF): J~3259-N13.}

\bibliographystyle{splncs}
\bibliography{RL}

\begin{appendix}
\section{Proof of Theorem~\ref{thm:aggsamples}}\label{app:aggsamples}
\subsubsection{Splitting into Episodes}
We follow the proof of Theorem~2 in \cite{jaorau}. 
First, as shown in Section 4.1 of \cite{jaorau},
setting $\Delta_k:= \sum_{s,a} v_k(s,a)  ( \rho^* -  \meanr(s,a))$
with probability at least $1-\tfrac{\delta}{12T^{5/4}}$ the regret after $T$ steps can be upper bounded by 
  \begin{equation}\label{eq:r-bound}
    \st\sum_{k=1}^m \Delta_k + {\sqrt{\tfrac{5}{8} T \log \left(\tfrac{8T}{\delta}\right)}}\;.
  \end{equation}
   
\subsubsection{Failing Confidence Intervals}
Concerning the regret with respect to the true MDP~$M$ being not contained in 
  the set of plausible MDPs~${\mathcal{M}_{k}}$, we cannot use the same 
  argument (that is, Lemma 17 in Appendix C.1) as in \cite{jaorau}, since the random
  variables we consider for rewards and transition probabilities are
  independent, yet not identically distributed.
  
  Instead, fix a state-action pair $(s,a)$, let $S(s,a)$ be the set of states $s'$ with $p(s'|s,a)>0$ and 
  recall that $\hat{r}(s,a)$ and $\hat{p}(\cdot|s,a)$ are the estimates for rewards and transition probabilities
  calculated from all samples of state-action pairs of the same color~$c(s,a)$. 
Now assume that at step $t$ there have been $n>0$ samples of state-action pairs of color~$c(s,a)$ and
that in the $i$-th sample action $a_i$ has been chosen in state $s_i$ and a transition to state $s'_i$ has been observed $(i=1,\ldots,n)$.
Then
\begin{eqnarray}
	\lefteqn{\Big\| \hat{p}(\cdot|s,a)- \mathbb{E}[\hat{p}(\cdot|s,a)] \Big\|_1  
	=
	\sum_{s'\in S(s,a)} \Big|  \hat{p}(s'|s,a) - \mathbb{E}[\hat{p}(s'|s,a)]  \Big|}  \nonumber
	\\
	&\leq&  
	\sup_{x\in\{0,1\}^{|S(s,a)|}} \sum_{s'\in S(s,a)} \Big(  \hat{p}(s'|s,a) - \mathbb{E}[\hat{p}(s'|s,a)]  \Big)\, x(s')  \nonumber\\
	&=&
	\sup_{x\in\{0,1\}^{|S(s,a)|}} \tfrac{1}{n}\sum_{i=1}^n 
			\Big( x(\phi_{s_i,a_i,s,a}(s'_i))  - \sum_{s'} p(s'|s_i,a_i)\cdot x(\phi_{s_i,a_i,s,a}(s'))  \Big)\,. \quad \label{eq:ah}
\end{eqnarray}
For fixed $x\in\{0,1\}^{|S(s,a)|}$, $X_i := x(\phi_{s_i,a_i,s,a}(s'_i)) - \sum_{s'} p(s'|s_i,a_i)\cdot x(\phi_{s_i,a_i,s,a}(s'))$ 
is a  martingale difference sequence 
 with $|X_i|\leq 2$, so that by Azuma-Hoeffding inequality (e.g., Lemma 10 in \cite{jaorau}),
$\Pr\{\vphantom{X^X_X} \sum_{i=1}^n X_i \geq \theta\} \leq  \exp( -  \theta^2/8n)$ and in particular 
\[
   \st \Pr\Big\{\vphantom{X^X_X} \sum_{i=1}^n X_i \geq \sqrt{56 Bn \log\big( \tfrac{4 t C_U}{\delta}\big)} \Big\} 
    \leq  \Big(\tfrac{\delta}{4 t C_U}\Big)^{7B}
    <  \tfrac{\delta}{2^B 20 t^7 C_U}.
\]
Recalling that by assumption $|S(s,a)|\leq B$, a union bound over all sequences $x\in\{0,1\}^{|S(s,a)|}$ then shows from \eqref{eq:ah} that
 \begin{eqnarray}\label{eq:probc}
   \st \Pr\left\{
      \Big\| \vphantom{X^X_X}
         \hat{p}(\cdot|s,a) - \mathbb{E}[\hat{p}(\cdot|s,a)]
      \Big\|_1
      \geq \sqrt{\frac{56B}{n}\log\left( 4 C_U t / \delta\right ) }
    \right\}
  &  \leq & \st\frac{\delta}{20 t^7 C_U}.
  \end{eqnarray}
  
Concerning the rewards, as in the proof of Lemma 17 in Appendix C.1 of~\cite{jaorau} --- but now using Hoeffding for independent and not necessarily identically distributed random variables --- we have that 
  \begin{eqnarray}\label{eq:rewc}
  \st  \Pr\left\{
      \left\vert \vphantom{X^X_X}
        \hat{r}(s,a) - \mathbb{E}[\hat{r}(s,a)]
      \right\vert
      \geq \sqrt{\frac{7}{2n}\log\left( 2 C_U t / \delta\right ) }
    \right\}
  & \leq & \st\frac{\delta}{60 t^7 C_U}.
  \end{eqnarray}
A union bound over all $t$ possible values for $n$ and all $C_U$ colors of states in~$\Psi_U$ 
shows that the confidence intervals in \eqref{eq:probc} and \eqref{eq:rewc} hold with probability at least $1-\frac{\delta}{15 t^{6}}$
for the actual counts $N(c(s,a))$ and all state-action pairs $(s,a)$. (Note that equations 
\eqref{eq:probc} and \eqref{eq:rewc} are the same for state-action pairs of the same color.)

By linearity of expectation, $\mathbb{E}[\hat{r}(s,a)]$ can be written as $\frac{1}{n} \sum_{i=1}^n r(s_i,a_i)$ for the sampled state-action
pairs $(s_i,a_i)$. Since the $(s_i,a_i)$ are assumed to have the same color~$c(s,a)$, it holds that $|r(s_i,a_i) -r(s,a)|<\veps$  and hence $|\mathbb{E}[\hat{r}(s,a)] - r(s,a)| < \veps$.
Similarly, $\big\|\mathbb{E}[\hat{p}(\cdot|s,a)]-p(\cdot|s,a)\big\|_1 < \veps$. Together with \eqref{eq:probc} and \eqref{eq:rewc}
this shows that with probability at least $1-\frac{\delta}{15 t^{6}}$ for all state-action pairs $(s,a)$ 
  \begin{eqnarray}
      \left\vert \vphantom{X^X_X}
        \hat{r}(s,a) - r(s,a)
      \right\vert
      &<& \veps + \st\sqrt{\frac{7}{2n}\log\left( 2 C_U t / \delta\right ) }  \label{eq:rewc2},  \\
       \Big\| \vphantom{X^X_X}
         \hat{p}(\cdot|s,a) - {p}(\cdot|s,a)
      \Big\|_1
      &<& \veps + \st \sqrt{\frac{56B}{n}\log\left( 4 C_U t / \delta\right ) }. \label{eq:probc2}
  \end{eqnarray} 
Thus, the true MDP is contained in the set of plausible MDPs $\mathcal{M}(t)$ at step $t$
 with probability at least $1-\frac{\delta}{15 t^{6}}$,
just as in Lemma 17 of~\cite{jaorau}. The argument that 
  \begin{equation} \label{eq:confidence}
     \st \sum_{k=1}^m \Delta_k \Ind{M \not\in {\mathcal{M}_{k}}}  \; \leq \; \sqrt{T}
  \end{equation}
   with probability at least $1-\frac{\delta}{12 T^{5/4}}$ then can be taken without any changes from 
   Section 4.2 of \cite{jaorau}.

\subsubsection{Episodes with $M\in{\mathcal{M}_{k}}$}\label{sec:confhold}
Now assuming that the true MDP $M$ is in $\mathcal{M}_{k}$, we first reconsider extended value iteration. 
In Section 4.3.1 of \cite{jaorau} it is shown that
for the state values $u_i(s)$ in the $i$-th iteration it holds that $\max_s u_i(s)-\min_s u_i(s)\leq D$, where $D$ is
the diameter of the MDP. Now we want to replace $D$ with the diameter $D_\veps$ of the aggregated MDP.
For this, first note that for any two states $s,s'$ which are aggregated we have by definition of the aggregated MDP
that $u_i(s) = u_i(s')$.
As it takes at most $D_\veps$ steps on average to reach any aggregated state, repeating the argument of Section~4.3.1 
of \cite{jaorau} shows that
\begin{equation}
    \st \max_s u_i(s)-\min_s u_i(s) \leq D_\veps.
\end{equation}

Let $\tmP_k:=\big(\tilde{p}_{k}(s'|s,\tpi{k}(s))\big)_{s,s'}$ 
    be the transition matrix of $\tpi{k}$ on $\tMk$, and 
  $\vv_k := \big(v_k\big(s,\tpi{k}(s)\big)\big)_s$ 
     the row vector of visit counts in episode~$k$ for each state and the corresponding
    action chosen by~$\tpi{k}$.
Then as shown in Sect.~4.3.1 of~\cite{jaorau}\footnote{Here we neglect the error by 
value iteration explicitly considered in Sect.\ 4.3.1 of~\cite{jaorau}.}
  \begin{eqnarray*}
    \Delta_k  &\leq &  
      \vv_k \big(  \tmP_k -\mI\big)\bw_k
        + \sum_{s,a} v_k(s,a) \big(  \tilde{r}_k(s,a) -\meanr(s,a)\big),
  \end{eqnarray*}
where $\bw_k$ is the normalized state value vector with $w_k(s):=u(s) - (\min_s u(s)-\max_s u(s))/2$, 
so that $\|\bw_k\|\leq \frac{D_\veps}{2}$. 
  Now for $(s,a)\in \Psi_K$ we have $\tilde{r}_k(s,a) = \meanr(s,a)$, while for $(s,a)\in \Psi_U$
 the term $\tilde{r}_k(s,a) -\meanr(s,a) \leq |\tilde{r}_k(s,a) -\hat{r}_k(s,a)| + |\meanr(s,a)-\hat{r}_k(s,a)|$ 
is bounded according to \eqref{eq:civR} and \eqref{eq:rewc2}, as we assume that $\tMk,M\in{\mathcal{M}_{k}}$.
 Summarizing state-action pairs of the same color we get 
\begin{eqnarray*}
\Delta_k & \leq & \vv_k \big(  \tmP_k -\mI\big)\bw_k
        + 2 \sum_{c\in C(\Psi_U)} v_k(c) \cdot \Big(\veps + \sqrt{ \tfrac{7\log\left( 2 C_U t_k / \delta \right )}
              {2\max\{1,N_k(c)\}} }  \Big),
\end{eqnarray*}
where $C(\Psi_U)$ is the set of colors of state-action pairs in $\Psi_U$.
Let $T_k$ be the length of episode $k$.
Then noting that $N_k'(c):=\max\{1,N_k(c)\} \leq t_k \leq T$ we get 
\begin{eqnarray}\label{eq:poisson3}
\Delta_k & \leq & \vv_k \big(  \tmP_k -\mI\big)\bw_k
        + 2\veps T_k + \sqrt{ 14\log\left( \tfrac{2 C_U T}{\delta}  \right )}
        \!\!\! \sum_{c\in C(\Psi_U)} \!\!\!\frac{ v_k(c)}{\sqrt{N'_k(c)} }.
\end{eqnarray}

\subsubsection{The True Transition Matrix}\label{sec:MainTerm}
Let $\mP_k:=\big(p(s'|s,\tpi{k}(s))\big)_{s,s'}$ be the transition matrix of $\tpi{k}$ in the true MDP~$M$. We split 
  \begin{eqnarray} 
  \vv_k \big(  \tmP_k -\mI\big)\bw_k 
        & = & \vv_k \big(  \tmP_k - \mP_k\big)\bw_k 
         + \vv_k \big(  \mP_k -\mI\big)\bw_k . \label{eq:vPl}
  \end{eqnarray}
By assumption $\tMk, M \in {\mathcal{M}_{k}}$,
so that using \eqref{eq:civ} and \eqref{eq:probc2}  the first term in~(\ref{eq:vPl}) can be bounded by (cf.~Section~4.3.2 of~\cite{jaorau})
  \begin{eqnarray}
   \lefteqn{ \vv_k}&& \big(  \tmP_k - \mP_k\big)\bw_k 
     \;\leq\;  
    \sum_{s,a} v_k\big(s,a\big) \cdot 
             \big\| \tilde{p}_{k}(\cdot|s,a) - p(\cdot|s,a) \big\|_1 \cdot \| \bw_k \|_\infty \nonumber\\
    &\leq &
     2 \sum_{c\in C(\Psi_U)} v_k\big(c) \cdot
    \st  \left( \veps +
      \sqrt{
        \frac{56 B \log\left( 4 C_U T / \delta \right )}{N'_k(c)}
      } \right) \cdot \st\frac{D_\veps}{2} \nonumber \\
    &\leq &
      \veps D_\veps\, T_k +  D_\veps \sqrt{56B\log\left( \tfrac{2 C_U T}{\delta}  \right )}
        \sum_{c\in C(\Psi_U)} \frac{ v_k(c)}{\sqrt{N'_k(c)} }
    , \qquad\label{eq:zwi}
  \end{eqnarray}
since --- as for the rewards --- the contribution of state-action pairs in $\Psi_K$ is 0.

Concerning the second term in (\ref{eq:vPl}), as shown in Section~4.3.2 of~\cite{jaorau}
one has with probability at least $1-\tfrac{\delta}{12T^{5/4}}$
\begin{eqnarray}
 \sum_{k=1}^m &\vv_k&(\mP_k - \mI) \bw_k\Ind{\M\in{\mathcal{M}_{k}}}
\leq  D_\veps \sqrt{\tfrac{5}{2}T\log\left(\tfrac{8T}{\delta}\right)} 
           +  D_\veps \,C_U \log_2\big(\tfrac{8T}{C_U}\big), \label{eq:counts}
\end{eqnarray}
where $m$ is the number of episodes, and the bound $m \leq C_U\log_2\left(8T/{C_U}\right)$
used to obtain \eqref{eq:counts}
is derived analogously to Appendix~C.2 of \cite{jaorau}.

\subsubsection{Summing over Episodes with $M\in{\mathcal{M}_{k}}$}
\label{sec:together}
To conclude, we sum \eqref{eq:poisson3} over all
  episodes with ${\M\in{\mathcal{M}_{k}}}$, using \eqref{eq:vPl}, \eqref{eq:zwi}, and \eqref{eq:counts},
  which yields that with probability at least $1 -\tfrac{\delta}{12 T^{5/4}}$
  \begin{eqnarray} 
   \lefteqn{   \sum_{k=1}^m \Delta_k\Ind{\M\in\sM_{k}}
       \leq  \label{eq:CiHold}
         D_\veps \sqrt{\tfrac{5}{2}T \log\left(\tfrac{8T}{\delta}\right)} 
        +  D_\veps\, C_U\log_2\big(\tfrac{8T}{C_U}\big) 
         + \veps (D_\veps +2) T }
      \nonumber  \\ & & \mbox{}       
        + \left(D_\veps \sqrt{56B\log\left( \tfrac{2 C_U B T}{\delta}  \right )}
        +  \sqrt{ 14\log\left( \tfrac{2 C_U T}{\delta}  \right )} \right)
        \sum_{k=1}^m \sum_{c\in C(\Psi_U)} \frac{ v_k(c)}{\sqrt{N'_k(c)} }
        .\qquad
  \end{eqnarray}
As in Sect.~4.3.3 and Appendix C.3 of \cite{jaorau}, one obtains 
$\sum_{c\in C(\Psi_U)} \sum_k  \frac{v_k(c)}{\sqrt{N'_k(c)} } \;\leq\;  \left(\sqrt{2}+1\right)\sqrt{C_U T}$.
Thus, evaluating \eqref{eq:r-bound} by
  summing $\Delta_k$ over all episodes, by \eqref{eq:confidence} and \eqref{eq:CiHold} the regret is upper bounded with probability 
  $\geq 1 -\tfrac{\delta}{4 T^{5/4}}$ by
  \begin{eqnarray*} 
     \lefteqn{}    && \sum_{k=1}^m  \Delta_k\Ind{\M\notin\sM_{k}}
        + \sum_{k=1}^m \Delta_k\Ind{\M\in\sM_{k}}
        +  \sqrt{\tfrac{5}{8} T \log\left(\tfrac{8T}{\delta}\right)} 
      \label{eq:combined1} \\
      && \leq  
           \sqrt{\tfrac{5}{8}T \log\left(\tfrac{8T}{\delta}\right)} + \sqrt{T} 
        + D_\veps\sqrt{\tfrac{5}{2}T \log\left(\tfrac{8T}{\delta}\right)}
        +  D_\veps\, C_U \log_2\left(\tfrac{8T}{C_U}\right)
       \nonumber  \\ & &
        + \veps(D_\veps+2)T
        + 3 \big(\sqrt{2}+1\big) D_\veps \sqrt{14B C_U T\log\left( \tfrac{2 C_U B T}{\delta}  \right )} 
        \,. \nonumber
  \end{eqnarray*}
Further simplifications as in Appendix C.4 of \cite{jaorau} finish the proof.\qed
\end{appendix}

\end{document}